\def\year{2018}\relax
\documentclass[letterpaper]{article} 
\usepackage{aaai18}  
\usepackage{times}  
\usepackage{helvet}  
\usepackage{courier}  
\usepackage{url}  
\usepackage{graphicx}  
\usepackage{natbib}

\usepackage{graphics} 
\usepackage{graphicx} 
\usepackage{epsfig} 
\usepackage{times} 
\usepackage{amsmath} 
\usepackage{amssymb}  
\usepackage{algorithm}
\usepackage{algorithmic}
\usepackage{subcaption}

\usepackage{bm}

\newtheorem{theorem}{Theorem}

\newtheorem{prop}{Proposition}

\usepackage{todonotes}

\begin{document}
%
\title{Guiding the search in continuous state-action spaces by learning an action
sampling distribution from off-target samples}
\author{Beomjoon Kim, Leslie Pack Kaelbling and Tom\'as Lozano-P\'erez
}
\maketitle

\begin{abstract}
In robotics, it is essential to be able to plan efficiently in high-dimensional 
continuous state-action spaces for long horizons. For such complex planning problems,
unguided uniform sampling of actions until a path to a goal is found is hopelessly
inefficient, and gradient-based approaches often fall short when the optimization 
manifold of a given problem is not smooth. In this paper we present an approach that
guides the search of a state-space planner, such as A*, by learning an action-sampling
distribution that can generalize across different instances of a planning problem. 
The motivation is that, unlike typical learning approaches for planning for continuous
action space that estimate a policy, an estimated action sampler is more robust to error 
since it has a planner to fall back on. We use a Generative Adversarial Network (GAN),
and address an important issue: search experience consists of a relatively large number
of actions that are not on a solution path and a relatively small number of actions that 
actually are on a solution path. We introduce a new technique, based on an importance-ratio
estimation method, for using samples from a non-target distribution to make GAN learning 
more data-efficient. We provide theoretical guarantees and empirical evaluation in three
challenging continuous robot planning problems to illustrate the effectiveness 
of our algorithm.
\end{abstract}

\newcommand{\Dq}{\mathbf{A}_q}
\newcommand{\Dp}{\mathbf{A}_p}
\newcommand{\D}{\mathbf{A}}
\newcommand{\hatD}{\hat{\mathbf{A}}}
\newcommand{\KL}{{\rm KL}}

\newcommand{\s}{s}
\newcommand{\act}{a}
\newcommand{\aspace}{\mathcal{A}}
\newcommand{\sspace}{\mathcal{S}}
\newcommand{\transfcn}{T}
\newcommand{\prob}{\omega}
\newcommand{\Prob}{\Omega}
\newcommand{\so}{s_0}
\newcommand{\sg}{s_g}
\newcommand{\valfcn}{V}
\newcommand{\psample}{q_{\act|\s}}
\newcommand{\ptarget}{p_{\act|\s}}
\newcommand{\ptargetset}{\mathcal{P}_{a|s}}
\newcommand{\pideal}{p^*_{\act|\s}}
\newcommand{\dlabel}{L}
\newcommand{\pg}{p_G}
\newcommand{\xp}{a_p^{(i)}}
\newcommand{\xq}{a_q^{(i)}}
\newcommand{\xg}{a_g^{(i)}}
\newcommand{\np}{n_p}
\newcommand{\nq}{n_q}
\newcommand{\ngen}{n_g}
\newcommand{\xw}{a_w^{(i)}}
\newcommand{\pw}{p_w}
\newcommand{\hatw}{\hat{w}}
\newcommand{\x}{a}
\newcommand{\suppq}{\mathcal{A}_q}
\newcommand{\suppp}{\mathcal{A}_p}

\section{Introduction}
The ability to efficiently plan in domains with
continuous state and action spaces  is a crucial 
yet challenging skill to obtain for a robot. 
For the classical path planning problem of getting
from an initial state to a goal state, the motion
planning community has found that random sampling
strategies or gradient-based approaches
work reasonably well~\citep{rrt,ZuckerIJRR13}.

In a variety of important planning problems, however, 
these approaches fall short. Consider the problem in 
Figure~\ref{fig:cycle_ex}, where the robot has to find 
a collision-free inverse kinematics solution
to reach the orange object by reconfiguring the green objects.
An unguided uniform sampling approach performs poorly since the state 
space is extremely high-dimensional, consisting of 
the combined configuration spaces of the robot and many objects.
A gradient computation is also difficult, since the robot 
has to make choices that are both discrete, 
such as which object to move, and continuous, 
such as where to place the chosen object, making the problem's
optimization manifold non-smooth. This type of hybrid 
optimization problem is difficult to solve efficiently.

\begin{figure}
  \centering
  \begin{subfigure}[t]{0.5\textwidth}
		  \centering
      \includegraphics[height=2cm]{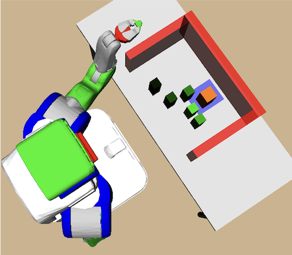}
		  \includegraphics[height=2cm]{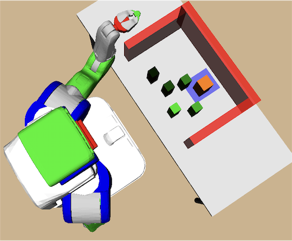}
      \caption{An example of the reconfiguration task. The movable obstacles are colored green, and the target object is colored with orange. }
		  \label{fig:cycle_ex}
  \end{subfigure}
  \begin{subfigure}[t]{0.5\textwidth}
      \centering
		  \includegraphics[height=2cm]{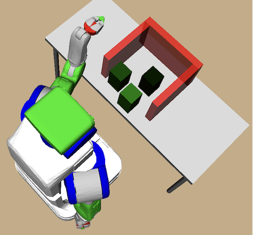}
		  \includegraphics[height=2cm]{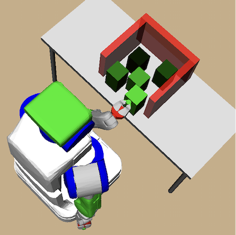}
      \caption{ An example of the bin packing task. The color of objects indicate 
the order of placement, and the darker the earlier.}
		  \label{fig:deadend_ex}
  \end{subfigure}
  \caption{Examples of reconfiguration and bin packing. For both tasks, the robot can only grasp objects from the side.}
\end{figure} 

A promising approach for such complex hybrid planning problems 
is to sample some predefined number of actions in each state, 
and then treat the resulting problem as a heuristic search
over discrete states and actions. 
A reconsideration strategy, which reconsiders states that 
has been previously expanded, can be employed to guarantee 
completeness.

The main distinction from a search in discrete state and action spaces is
that, in the discrete case, having a good heuristic
function is sufficient for efficiently finding a solution. A perfect
heuristic, for instance, would find a plan without any exploration.
However, in the continuous state and action spaces this is not true;
if there is a very small volume of actions that lead to the goal, then
even with a perfect heuristic we might inefficiently sample actions
and reject most of them. Moreover, if the heuristic is not perfect,
as is typically the case, having samples in undesirable regions of
the action space could lead to expanding nodes that do not efficiently
lead to a goal state.

Based on this observation, our objective in this paper is to learn 
from search 
experience an action sampling distribution  that will guide the search of a planner
that has an imperfect heuristic. 
We use a generative adversarial network (GAN), 
a recent advance in generative model learning, to approximate this 
distribution~\citep{GoodfellowNIPS2014}.
Unlike other methods, a GAN only requires a forward pass through a neural 
network to generate a sample, and does not require any prior
distributional assumptions.
The main challenge in learning a generative model from search
experience data is that in a successful 
episode of search, there is a large number of state and action
pairs that we sample but are not known to be on a trajectory to the goal,
and only a relatively small number of samples that are on a 
trajectory to the goal. We will refer to samples on a
successful trajectory as \emph{on-target} samples, to indicate
that they are from a target distribution that we would like to learn,
and the rest as \emph{off-target} samples.
While we could just use the small number of on-target samples, 
learning would be much more effective 
if we can find a way to use the abundant off-target samples as well.

In light of this, we propose a principled approach that can
learn a generative model from both on- and off-target samples.
To do this, we estimate the importance-ratio between the on-target 
and off-target distributions using both classes of
samples. We then extend GAN to introduce a new 
generative-model learning algorithm, called 
generative adversarial network with direct importance estimation (GANDI),
that uses the importance ratio to learn from
not only on-target samples, but also from off-target samples.
While this algorithm is domain independent, 
we will demonstrate its effectiveness in learning a target
action-sampling distribution.
We theoretically analyze how the importance-ratio estimation 
and the difference between target and non-target distributions
affect the quality of the resulting approximate distribution. 

We evaluate GANDI in three different planning problems
that commonly occur in warehousing applications.
We show that our algorithm outperforms a 
standard uniform sampler and a standard GAN in terms of 
planning  and data efficiency.

\section{Related work}
Our work touches upon four different topics: continuous state-action 
space planning, learning to guide planning, importance-ratio 
estimation, and generative model learning. 
We provide descriptions of approaches in these 
areas that are closest to our problem in terms of motivation and technique.

In the area of continuous-space planning, there are several approaches
that attempt to sample actions by employing optimistic optimization
methods for black-box functions. Mansley
et. al~\citeyearpar{MansleyICAPS2011} assign a hierarchical optimistic
optimization (HOO)~\citep{BubeckNIPS2009} agent at each state to
decompose the action space to sample promising ones. A similar
approach is taken by Bus\c{o}niu
et. al,~\citeyearpar{LucianADPRL2011}, where a simultaneous optimistic
optimization (SOO) agent is assigned at each
state~\citep{MunosNIPS11}. There are other approaches to more
specialized continuous state and action problems, such as
motion-planning and task-and-motion planning; these include
sample-based methods~\citep{rrt,prm,GarrettRSS2017,VBrownWAFR2016},
and gradient-based optimization methods
~\citep{ZuckerIJRR13,SchulmanIJRR14,ToussaintIJCAI15}.  The main
objective of all these approaches is to construct, on-line, an optimal
or a feasible plan without any off-line learning. No knowledge of
solving one problem instance is transferred to another.

In learning to guide search, there is a large body of work 
that attempts to learn a policy or a value function off-line,
 and then use this information during an on-line search 
to make planning efficient. These methods are usually applied to
discrete-action problems, where a recent prominent example is
 AlphaGo~\citep{alphago}. 
In that paper, Silver et. al train a supervised policy off-line, 
using imitation 
learning and train a value function using self-play; they then guide 
Monte Carlo Tree Search (MCTS) in an on-line phase using these functions.
In a similar line of work, Gelly et al.~\citeyearpar{GellyICML07} also integrates
an off-line learned value function with MCTS. For
learning to guide search in continuous-space planning problems, 
Kim et. al~\citeyearpar{KimICRA17}  describe an approach for 
predicting constraints on the solution space rather than a value 
function or an action itself.
The intuition is that a constraint is much more easily 
transferable across problem instances than a specific action in complex 
planning problems. We share this intuition, and we
may view the learned action distribution as constraining the search
space of a planning problem to promising regions.

Two recent advancements in generative-model learning, 
GANs~\citep{GoodfellowNIPS2014} 
and Variational Auto Encoders (VAEs)~\citep{VAE}, are 
appealing choices for learning an action-sampling 
distribution because inference is 
simply a feed-forward pass through a network.
GANs are especially appealing,
because for generic action spaces, we do not have any metric information
 available. VAEs, on the other hand, require a metric in an action space in 
order to compute the distance between a decoded sample and a true sample. Both
of these methods require samples  to be from a target distribution
that one wishes to learn. Unfortunately, in our case we have limited access to
samples from the target distribution, and this brings us to the importance-ratio
estimation problem.

There is a long history of work that uses importance sampling to 
approximate desired statistics for a target distribution $p$ using 
samples from another distribution $q$,
for example, in probabilistic graphical models~\citep{KollerFriedman2009} and
reinforcement learning problems~\citep{PrecupICML2001,SuttonBarto1998}.
In these cases, we have a surrogate distribution $q$ that is cheaper 
to sample than the target distribution $p$. 
Our work shares the same motivation as these problems, 
in the sense that in search experience data, samples that are on 
 successful trajectories are expensive to obtain, while other 
samples are relatively cheaper and more abundant.

Recently, importance-ratio estimation has been studied for the 
problem of covariate shift adaptation, which closely resembles our setting. 
Covariate shift refers to the situation where we have samples from 
a training distribution that is different from the target distribution. 
Usually, an importance-ratio
estimation method~\citep{KanamoriJMLR2009,SugiyamaNIPS2008,HuangNIPS2007} is 
employed to re-weight the samples from the training
distribution, so that it matches the target distribution 
for supervised learning.  We
list some prominent approaches here.
In kernel-based methods for estimating the importance~\citep{HuangNIPS2007},
authors try to match the mean of samples of $q$ and $p$ in a
feature space, by re-weighting the samples from $q$. In the direct estimation
approach, Sugiyama et al.~\citeyearpar{SugiyamaNIPS2008} try to minimize the
KL divergence between the distributions $q$ and $p$ by re-weighting $q$. In
another approach, Kanamori et al.~\citeyearpar{KanamoriJMLR2009} directly minimize 
the least squares objective between the approximate importance-ratios and the
 target importance-ratios. All these formulations yield convex
optimization problems, where the decision variables are parameters of a linear
function that computes the importance weight for a given random variable. Our
algorithm extends the direct estimation approach using a deep neural network, 
and then applies it for learning a generative model using off-and-on target samples.

\section{Background}
\subsection{Planning in continuous spaces}
A deterministic continuous action and state space planning
 problem is a tuple $[\sspace, \aspace,s_0, G, T]$ where 
$\sspace$ and $\aspace$ are state and action spaces,
 $T:\sspace \times 
\aspace \rightarrow \sspace$ is a transition function that
maps a state and an action to a next state, $s_0 \in \sspace$ is the
initial state and $G \subset \sspace$ is a goal set.

An instance of a planning problem consists of a tuple $(s_0,G,\prob)$, 
where $\prob\in \Omega$ represents parameters of a problem instance. 
While the state changes according to $T$ when an action is executed, 
the parameters represent aspects of the problem that do not change within
the given instance. For example, a state might represent
poses of objects to be manipulated by a robot,
 while $\prob$ might represent their
shapes.  Given a planning problem instance, a heuristic state-space planner
requires a heuristic function $H:\sspace \rightarrow \mathbb{R}$
that estimates a cost-to-go. This function might be learned, or
designed by a user. The planning algorithm shown in 
Algorithm~$\ref{alg:planalgo}$ describes an extension of
discrete-space heuristic planner to continuous-spaces. This version 
of the algorithm is greedy with respect to
$H$, but it is straightforward to arrange for a version that is more
like A*, by taking path cost to the current state into account as well.

The key distinction from a discrete state-action space heuristic search 
algorithm is the action-sampling distribution $\psample$ and the 
reconsideration strategy for a node.
While in the discrete case we push all the neighbors of the current
node, here we sample $k$ actions and then push the resulting nodes.
Without any domain-specific knowledge, $\psample$ is a uniform distribution
in some bounded region in $\aspace(s,\prob)$ that depends on the 
current state and the parameters of the problem. 
The reconsideration strategy refers to 
pushing the node that we just popped  back to the queue, in order to 
guarantee probabilistic completeness.

Given this setup, we can see that even a perfect heuristic will not be
helpful unless good action samples are generated. Moreover, in most
cases we have an imperfect heuristic function. This motivates the
problem of learning an action-sampling distribution, which we formulate next
by considering what information we can extract from search experience.

\begin{algorithm}[tb]
\small
 \caption{Search($s_0,G,\transfcn,k, H, \psample$)}
 \label{alg:planalgo}
\begin{algorithmic}
\STATE $Q = {\it InitQueue}({\it Node}(s_0))$
\WHILE{True}
\STATE $n= Q.pop(\arg\min_{n\in Q} H(n.s))$ //dequeue the cheapest $n$
\IF {$n.s \in G$}
\STATE {\bf return} $n$
\ENDIF
\FOR{$i \in 1..k$} 
\STATE Sample $a \sim \psample(a \mid n.s)$
\STATE $Q.push( Node(T(n.s,a) )$
\ENDFOR
\STATE $Q.push(n)$ // reconsideration
\ENDWHILE
\end{algorithmic}
\end{algorithm}


\subsection{Problem formulation}
Ideally, our objective would be to learn a distribution that
assigns zero probability to actions that are not on an
optimal trajectory to $G$, and non-zero probability to actions on 
an optimal trajectory. Such a distribution would yield a path to
a goal without any exploration, regardless of the quality of the given
heuristic function.

Unfortunately, in sufficiently difficult problems, optimal planners are 
not available, therefore it is 
impossible to determine  whether an action was on an optimal 
path from the search-experience data. Thus, we consider
an alternative objective: learn a distribution that 
assigns low probability to actions
that lead to undesirable parts of the state space.
 We now examine two types of undesirable actions that lead
to such states.

The first type is {\em dead-end actions} that lead to {\em dead-end states}. 
A state $s$ is a dead-end state if there is no feasible path
from $s$ to a state in $G$.  Dead-end states are clearly to be avoided if
possible, because all search effort below them in the search tree is
guaranteed to be wasted. An example of a dead-end state is shown in
Figure~\ref{fig:deadend_ex}. Here, the robot is asked to pack
five objects into the relatively tight bin. It cannot move an object
once it is placed, so if the robot places a first few objects
in the front of the bin like in Figure~\ref{fig:deadend_ex} (left),
then it will be in a dead-end state. 

The second type is {\em no-progress actions} that transition to a state
that has the same or greater true cost-to-go as the current state.
 An example is shown in Figure~\ref{fig:cycle_ex}, where 
the robot has to reconfigure
poses of the green objects to find a collision-free space to
reach the orange object with its left arm. When the robot moves 
the light green object as shown in Figure~\ref{fig:cycle_ex} 
(left) to (right), this action does not constitute progress towards making room for 
reaching the target object.

Now, our objective is to learn an action distribution $\ptarget$ that assigns
low probabilities to these two types of actions.
We denote the $m$ search experience episodes for training data,
 collected using a
domain-independent action sampler $\psample$, as
$\Big\{ \Big[(s_i,\xp)_{i=1}^{n_p},(s_i,\xq)_{i=1}^{n_q},(\prob^{(j)},s_0^{(j)},G^{(j)}) \Big] \Big\}_{j=1}^{m}  $
where $\xp$ is an action on the trajectory from $s_0^{(j)}$ to $s\in G^{(j)}$, 
$\xq$ is an action in the search tree, but not the solution
trajectory, and $s_i$ is the state in which $\xp$ or $\xq$ was executed.
$\np$ and $\nq$ denote the number of state and action pairs that
are on the trajectory to the goal and the rest, respectively.

Fortunately, the distribution of $(s,a)$ pairs on successful
trajectories in the data has the following properties: they have
zero probability assigned to dead-end states and actions, dead-end
actions and states cannot occur on a path to the goal. They also have
low probability assigned to no-progress actions, because most
actions, though not necessarily all, are making progress.  However, we
cannot say anything about $(s,a)$ pairs not on a successful trajectory:
they may or may not be desirable.  Therefore, we will call the $\xp$
values \emph{on-target} samples, and the $\xq$ as \emph{off-target}
samples. Our algorithm, GANDI, uses both of these two sources of data
to learn a target distribution.

\subsection{Generative Adversarial Networks}
The objective of a GAN is to obtain a
generator $G$ that transforms a random variable $z$, 
sampled usually from a Gaussian or a uniform distribution, to
a random variable of interest which in our case is an action $a$.
A GAN learns $\pg$, an approximation of the target distribution, which
in our case is $\ptarget$. For the purpose of
description of GAN, we will treat the distributions 
as unconditional; they can be directly extended to condition on 
$\s$ by viewing the transformation as mapping $(\s,z)$ to $\act$.

We denote the on-target samples
$\Dp := \{\xp\}_{i=1}^{\np}$, and the samples generated by $G$ as 
$\{\xg\}_{i=1}^{n_g}$. To learn the generator $G$, a GAN alternates
between optimizing a function $D$, which tries to 
output 1 on on-target samples and output 0 on samples generated by $G$,
and optimizing $G$, which tries
 to generate samples that cause $D$ to output 1. 
This leads to the following objectives:
\begin{align}
D  &= \arg\min_D\sum_{i=1}^{\np} \log D(\xp) + \sum_{i=1}^{\ngen} \log(1-D(\xg))
\label{eq:ganD}\\
G  &= \arg\max_G \sum_{i=1}^{\ngen} \log D(\xg)
\label{eq:ganG}
\end{align}
where $\np=\ngen$ to force $D$ to assume the classes are equally likely.
After training, given a sample of $z \sim p_z$, the neural
 network $G(z)$ outputs $\x$ with probability
$\pg(\x)$.

\subsection{Direct importance ratio estimation}
We now describe an approach to estimate  importance weights between the 
target and uniform distribution,
$
w(\x;\s) = \ptarget(\x|\s)/\psample(\x|\s)
$ using a deep neural network (DNN).  If we had these weights, then
we could use $\Dq$ to augment $\Dp$ used for training the 
generative model $\pg$ to approximate $\ptarget$. While
there are several methods for estimating 
$w(\x)$~\citep{HuangNIPS2007,KanamoriJMLR2009,SugiyamaNIPS2008}, 
we will use the least squares
approach of Kanamori et al.~\citeyearpar{KanamoriJMLR2009} because it integrates easily 
into DNN-based methods. In this approach, we approximate $w$ with $\hatw$ using 
the objective function 
\begin{align*}
J(\hat{w}) &= \int_\x (\hat{w}(\x)-w(\x))^2 q(\x) d\x \, .
\end{align*}
In practice, we optimize its sample approximation version, $\hat{J}(\hatw)$, 
which yields
\begin{align}
\hat{w} 
&= \arg\min_{\hat{w}} \sum_{i=1}^{\nq} \hat{w}^2(\xq) - 2\sum_{i=1}^{\np} \hat{w}(\xp), \text{ s.t } \hatw(\x) \geq 0 \label{eq:lsif}
\end{align}
where the constraint is enforced to keep the importance weights positive.
Intuitively, $\hatw$ attempts to assign high values  to $\xp$ and 
low values $\xq$, to the degree allowed by its hypothesis class. 

The method was originally proposed to be used with a linear architecture, 
in which $\hatw(\x) = \theta^T\phi(\x)$; this implies there is a 
unique global optimum as a function of $\theta$, but 
requires a hand-designed feature 
representation $\phi(\cdot)$.
For robot planning problems, however,
 manually designing features is
difficult, while special types of DNN architectures, 
such as convolutional neural networks, may effectively
learn a good representation. 
In light of this, we represent $\hatw$ with a DNN. 
The downside of this  strategy is the loss of
convexity with respect to the free parameters $\theta$, but we have found that the 
flexibility of representation offsets this problem.


\section{Generative Adversarial Network with Direct Importance Estimation}
\label{sect:algo}
In this section, we introduce our algorithm, GANDI, 
which can take advantage of 
off-target samples from $\psample$ using importance weights. 
We first describe how to formulate the objective for training 
GANs with importance weights. For the purpose of exposition, we begin by 
assuming we are given $w(\x)$, and we only have samples 
from the off-target distribution, $\Dq$, and none from the 
target distribution. Using importance weights $w(\x)$,  
the objective for the discriminator in GANs becomes
\begin{align}
D&= \arg\min_D \sum_{i=1}^{\nq} w(\xq)\log D(\xq) + \sum_{i=1}^{\ngen} \log (1-D(\xg))
\label{eq:igan}
\end{align}
In trying to solve the equation~\ref{eq:igan},  it is
 critical to have balanced training set sizes $\np$ and $\ngen$.
In the importance weighted version of the GAN shown in 
equation~\ref{eq:igan}, however, the sum 
of the weights  $c = \sum_{i=1}^{n_q} w(\xq)$, serves as an 
{\em effective sample size} for the data $\Dq$.  
To achieve a balanced objective, we might then select $\ngen$ 
to be equal to c.  

Taking this approach would require adjusting the the GAN objective 
function and modifying mini-batch gradient descent algorithm to
use  uneven mini-batch sizes in every batch in training,
 which is awkward in many
 high-performance NN software packages.  Instead,
 we develop a method for bootstrapping $\Dq$ that allows us to use existing
 mini-batch gradient descent codes without modification.
Specifically, instead of multiplying each off-target sample by its importance weight, 
we bootstrap (i.e. re-sample the data from $\Dq$ with replacement), 
with a probability $\pw(\x)$, where
$
\pw(\x) = \frac{w(\x)}{\sum_{i=1}^{\nq} w(\xq)} \, .
$
This method allows us to generate a dataset $\hatD$ in which the 
expected number of instances of each element of 
$\D$ is proportional to its importance weight.
Moreover, since we bootstrap, the amount of training data remains the same, 
and $D$ now sees a balanced number of samples effectively drawn from 
$p(\x)=w(\x)q(\x)$ and $\pg$.

One can show that $\pw$ is actually proportional to $p$.
\begin{prop} For $ \x \in \Dq$,
$$p_w(\x) = k\cdot p(\x)\text{ where } k=\frac{1}{ \sum_{i=1}^{n_q}w(\xq)} \, .$$ 
\end{prop}
This means that samples drawn from the importance-reweighted data $\D$
 can play the role of $\Dp$ in the GAN objective.

We now describe some practical details for approximating $w(\x)$ with
$\hatw(\x)$, whose architecture is a DNN. Equation \ref{eq:lsif} can be
readily solved by a mini-batch gradient-descent algorithm implemented
using any readily available NN software package.  The non-negativity
constraint can also be straight-forwardly incorporated by simply using
the rectified linear activation function at the output
layer\footnote{In practice, this often lead the weights to
  converge to 0.  Although this can be avoided with a careful
  initialization method, we found that it is effective to just
  use linear activation functions, and then just set $w(\x)=0$ if
  $w(\x)<0$.}.  Now, with estimated importance weights and bootstrapped
samples, the objective for $D$ shown in equation~\ref{eq:igan} becomes
\begin{align}
 \hat{D} = \arg\min_D\sum_{i=1}^{\nq} D(\xw) + \sum_{i=1}^{\ngen} \log(1-D(\xg)) \label{eq:gandiD} \, ,
\end{align}
 where $\xw$ denotes a bootstrapped sample from $\Dq$, and $\nq=\ngen$.
This involves just using $\Dq$, but 
in practice, we also use $\Dp$, by simply augmenting the  dataset $\Dq$ with 
$\D:=\Dq \cup \Dp$, and then applying $\hatw(\cdot)$ to $\D$ for bootstrapping, yielding 
final dataset $\hatD$.  Algorithm~\ref{alg:GANDI}
contains the code for GANDI.

We illustrate the result of the bootstrapping with a simple example, shown 
in Figure \ref{fig:toyex}, where we have a Gaussian mixture model for both 
on-and off-target distributions $p$ and $q$, where $p$ is a mixture of two
Gaussians centered at $(1,1)$ and $(3,1)$, and $q$ is a mixture of three 
Gaussians at $(1,1), (3,1),$ and $(2,2)$ 
with larger variances than those of $p$.
\begin{figure*}
\centering
\begin{subfigure}[t]{0.3\textwidth}
  \centering
  \includegraphics[width=0.6\linewidth,height=2cm]{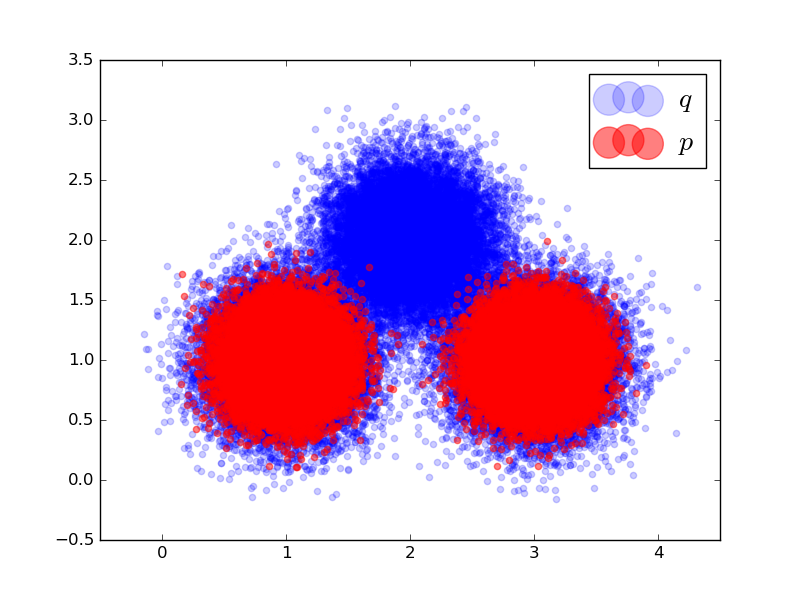}
  \caption{ Distributions $p$ and $q$}
\label{fig:pq}
\end{subfigure}
\begin{subfigure}[t]{0.3\textwidth}
  \centering
  \includegraphics[width=0.6\linewidth,height=2cm]{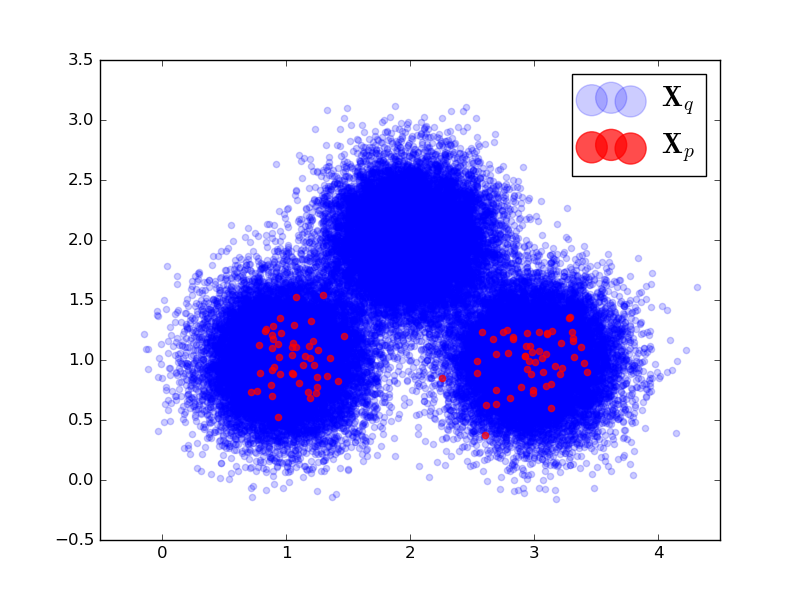}
  \caption{ Off- and on-target samples }
\end{subfigure}\\
\begin{subfigure}[t]{0.3\textwidth}
  \centering
  \includegraphics[width=0.6\linewidth,height=2cm]{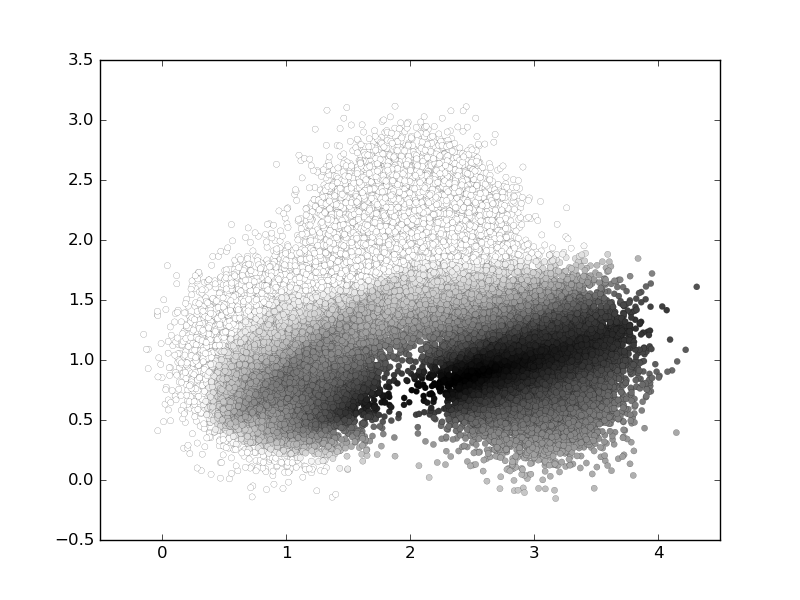}
  \caption{Importance weights of $\D$.}
  \label{fig:trainset_toy}
\end{subfigure}\hfill
\begin{subfigure}[t]{0.3\textwidth}
  \centering
  \includegraphics[width=0.6\linewidth,height=2cm]{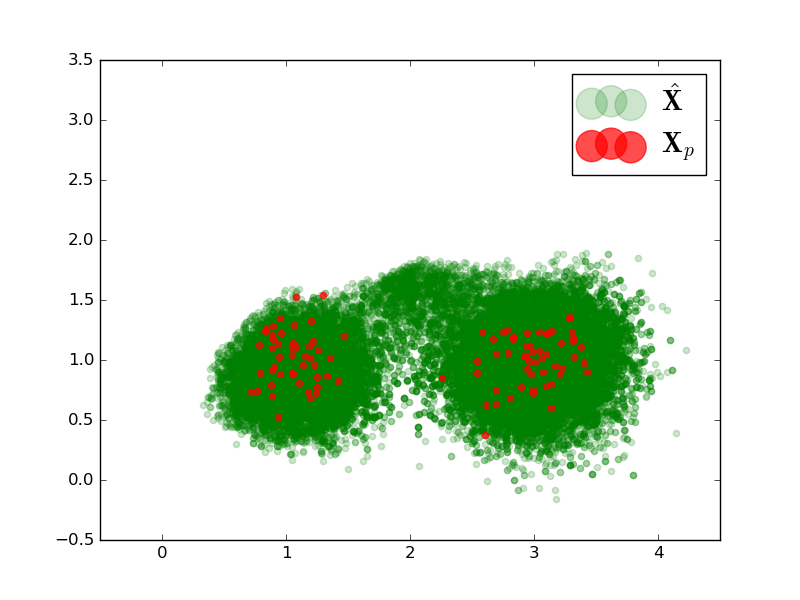}
  \caption{ Bootstrapped samples  }  
\end{subfigure}\hfill
\begin{subfigure}[t]{0.3\textwidth}
  \centering
  \includegraphics[width=0.6\linewidth,height=2cm]{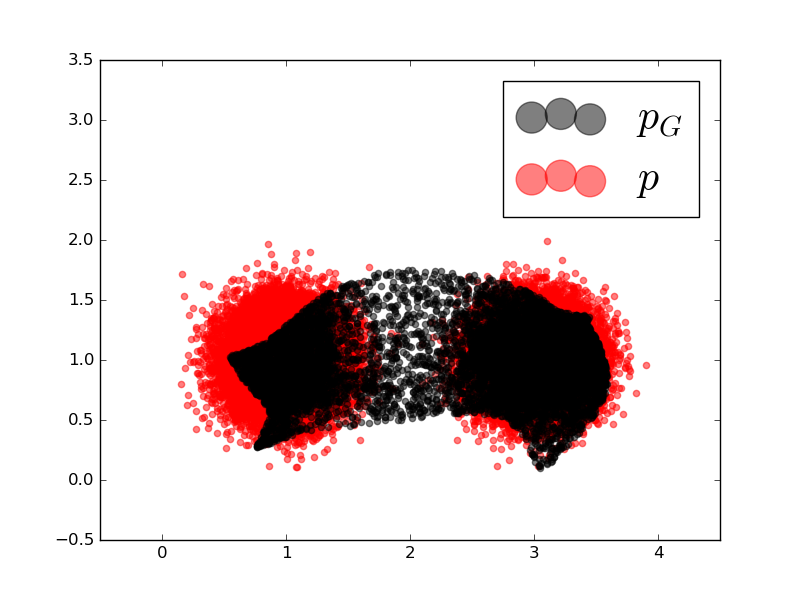}
  \caption{  $\pg$ and the target distribution $p$. }
\end{subfigure}
\caption{Example of importance weight estimation (c), bootstrapping (d),
 and samples from $\pg$ and $p$ (e) in an illustrative domain with Gaussian mixtures. }
\label{fig:toyex}
\end{figure*} 

\begin{algorithm}[tb]
\small
   \caption{GANDI($\Dp,\Dq$)}
   \label{alg:GANDI}
\begin{algorithmic}
\STATE $\hatw \leftarrow  EstimateImportanceWeights(\Dp,\Dq)$ // objective (\ref{eq:lsif}) with on-and off-target samples $\Dp$ and $\Dq$ \\
\STATE $\pw(\x) := \frac{\hatw(\x)}{\sum_{i=1}^{\nq}\hatw(\xq) + \sum_{i=1}^{\np}\hatw(\xp) }$ // define bootstrap probability distribution
\STATE $\D \leftarrow \Dp \cup \Dq$ 
\STATE $\hatD \leftarrow Bootstrap(\D, \pw)$ // sample $\D$ with probability $\pw$, with replacement
\STATE $G \leftarrow TrainGAN(\hatD)$ // train the GAN using objectives (\ref{eq:ganD}) and (\ref{eq:ganG}) with $\hatD$ as on-target samples
\STATE{\bf return} $G$
\end{algorithmic}
\end{algorithm}

\section{Theoretical analysis}
\label{sect:theory}
In this section, we analyze how error in importance estimation 
affects the performance of $\pg$ in approximating $p$.
The basic result on GANs, shown in the limit of infinite data,
representational and computational capacity, is that $\pg$ converges to
$p$~\citep{GoodfellowNIPS2014}, although subsequent papers have 
presented more subtle form of analysis~\citep{Arjovsky2017}.

Now, under the same assumptions, we consider the effect of using
importance weighted off-target data. If $w$ is exact, then $p(\x)=w(\x)q(\x)$
and the GAN objective is unchanged. If, however, we use an estimation of
importance weighting function $\hatw$, then the objective of $\hat{D}$,
the importance-weight corrected discriminator, 
differs from $D$ and they achieve different solutions. 

We wish to analyze the effect of importance estimation error on KL and
reverse-KL divergence between $p$ and $\pg$. First,
define $\rho = \sup_{\x \in \suppp} q(\x)/p(\x)$,
where $\suppp$ is the support of $p$.
We can see that $\rho >= 1$, with equality 
occurring when $p(\x) = q(\x)$ for all $\x$.

For the KL divergence, we have the following theorem.
\begin{theorem} If $w(\x) \geq \epsilon\ \forall \x\in \suppq$,
 and $J(\hatw) \leq \epsilon^2$, then 
$$ \KL(p||\pg) \leq \log \Big(\frac{1}{1-\epsilon\rho}\Big)\;\;.$$ 
\end{theorem}
Note that $0\leq \epsilon \rho \leq 1$ due to the
 condition $w(\x) \geq \epsilon$.
For reverse KL we have:
\begin{theorem}
If $J(\hatw) \leq \epsilon^2$,
$KL(p_G || p) \leq (1 + \epsilon) \log (1 + \epsilon\rho)$\;.
\end{theorem}
The proofs are included in the supplementary material.

These theorems imply three things:
(1) If $w=\hatw$, then $\epsilon=0$, and both divergences go to 0, regardless of $\rho$;
(2) If $p=q$, then the error in importance weight estimation is the only
source of error in modeling $p$ with $\pg$. This error can be arbitrarily large, as
$\epsilon$ becomes large; and
(3) If $p\neq q$ then $\rho>1$, and it contributes to the error in modeling $p$ with $\pg$.

\section{Experiments} \label{sec:experiments}

We validate GANDI on three different robot planning 
tasks that involve continuous state and action spaces and finite depths.
The purpose of these experiments is to verify the hypothesis
that learning an action sampling distribution
improves planning efficiency relative to a standard uniform sampler,
and a GANDI is more data efficient than a standard GAN that is
only trained with on-target samples.

We have three tasks that might occur in a warehouse.
The first is a bin-packing task, where a robot has to
pack different numbers of objects of different sizes into a
relatively tight bin. The second task is to stow objects into
crowded bins, where there already are obstacles. The final
task is a reconfiguration task,
where a robot needs to reconfigure five randomly placed
moveable objects in order to make room to reach
a target object. In the first two tasks, the robot is not 
allowed to move the objects once they are placed, which leads
 to a large volume of dead-end 
states that cause wasted computational effort for a 
planner with an uniform action sampler. In the third task,
we have no dead-end states, but a planner could potentially
waste computational effort in trying no-progress actions.
For all tasks, the robot is only allowed to grasp
objects from the side; this is to simulate a common scenario
in a warehouse environment, with objects in 
a place covered on top, such as a shelf.
For the first two experiments, we use depth-first-search. This
is equivalent to having a heuristic that estimates a cost-to-go
by number of objects placed, since we cannot move objects
once they are placed. For the last experiment we use 
breadth-first-search as our planner, since there is no
simple heuristic function evident. For all
cases, the number of action samples per node, $k$, is 3.
We use Algorithm~\ref{alg:planalgo} with a given action
sampler.

Throughout the three tasks, we compare three different 
action samplers in terms of success rate within a given
time limit: the uniform sampler that uniformly samples
an action from a specified action space, a standard GAN
 trained only with on-target samples, and GANDI, which is 
trained with both on and off-target samples. We use the same
architecture for both the standard GAN and GANDI, and
perform 100 repetitions to obtain the 95\% confidence intervals
for all the plots. The architecture description is in
the appendix.

A crucial drawback of generative adversarial networks
is that they lack an evaluation metric; thus it is
difficult to know when to stop training. We
deal with this by testing weights from all epochs on
10 trials, and then picking the weights with the 
best performance, with which we performed 100 additional
repetitions to compute the success rates.
\begin{figure*}
\centering
	\begin{subfigure}[b]{0.23\textwidth}	
		\centering
		\includegraphics[width=2.5cm,height=2.5cm]{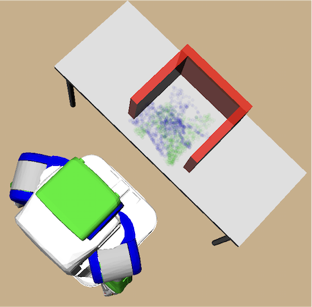}
		\caption{GAN's action sampling distribution }
		\label{fig:gan_ex}
	\end{subfigure} 
	\begin{subfigure}[b]{0.23\textwidth}	
		\centering
		\includegraphics[width=2.5cm,height=2.5cm]{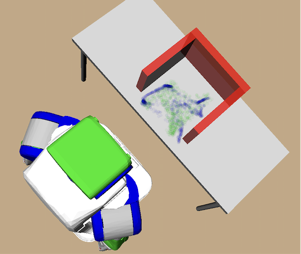}
		\caption{GANDI's action sampling distribution}
		\label{fig:gandi_ex}
	\end{subfigure} 
  \begin{subfigure}[b]{0.23\textwidth}	
		\centering
		\includegraphics[width=2.5cm,height=2.5cm]{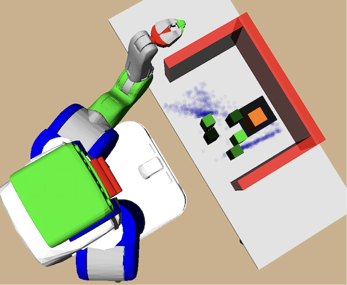}
		\caption{A GANDI's action distribution }
		\label{fig:rearrangement_ex}
	\end{subfigure} 
  \begin{subfigure}[b]{0.23\textwidth}	
		\centering
		\includegraphics[width=2.5cm,height=2.5cm]{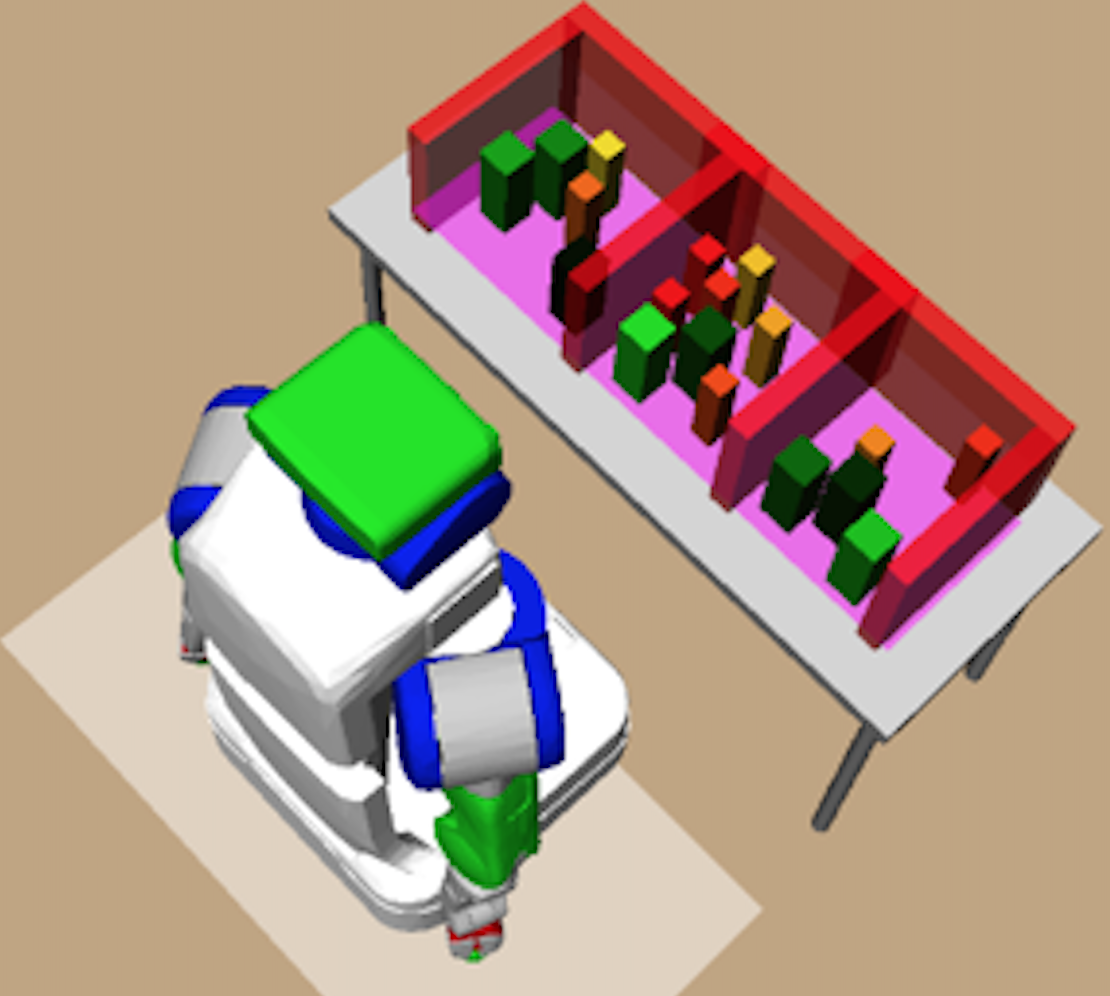}
		\caption{A stow domain problem instance}
		\label{fig:stow_ex}
	\end{subfigure} 
		\caption{The left two figures show examples from the bin packing domain when 20 episodes of training data are used. Green indicates the on-target samples, and blue indicates the learned distributions. Figure \ref{fig:rearrangement_ex} shows an action distribution for the reconfiguration domain when given 35 training episodes.}
\end{figure*} 
\begin{figure*}
\centering
	\begin{subfigure}[b]{0.23\textwidth}	
		\centering
		\includegraphics[height=2.5cm]{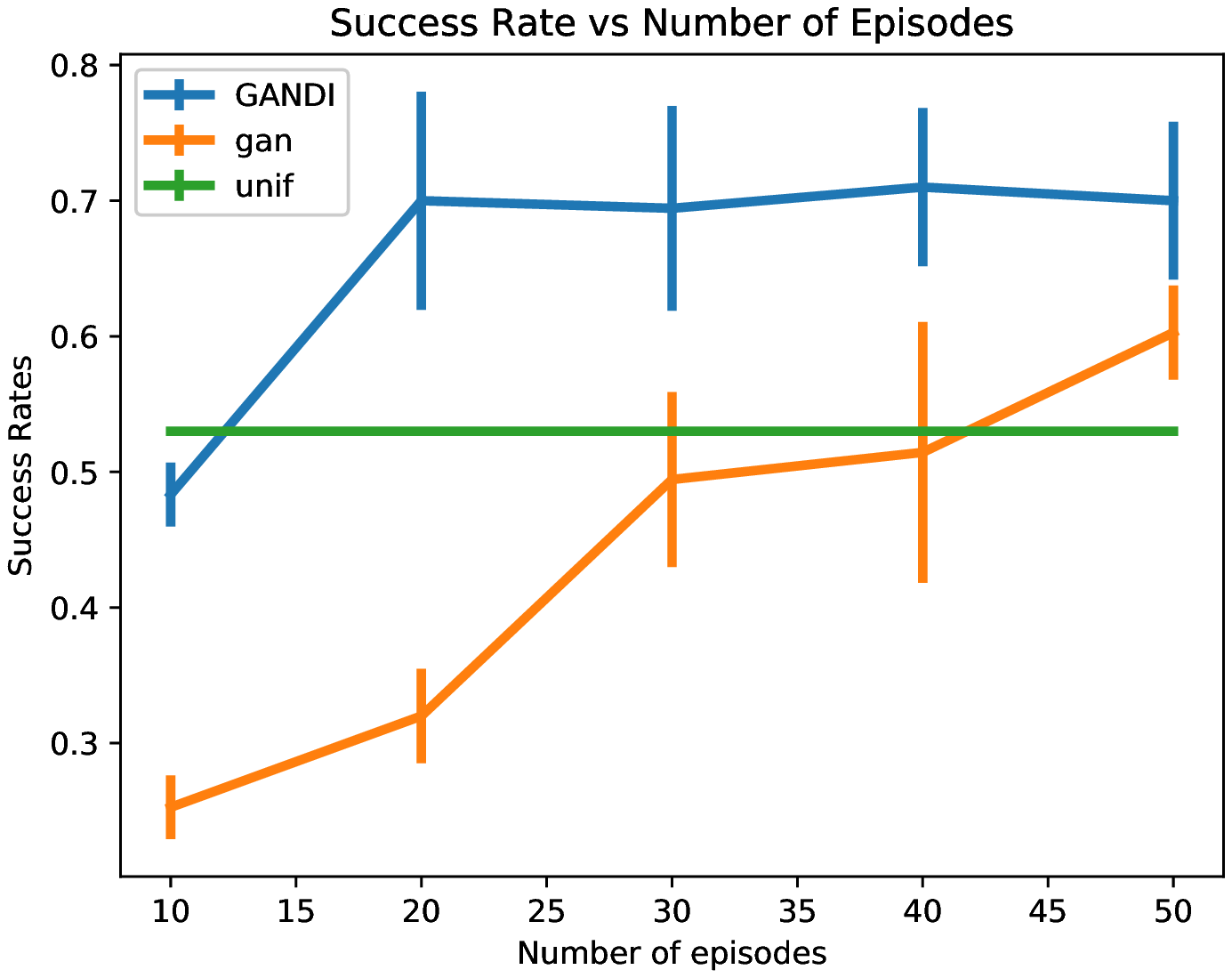}
    \caption{The bin packing domain}
		\label{fig:bin_packing_rate}
	\end{subfigure} 
	\begin{subfigure}[b]{0.23\textwidth}	
		\centering
		\includegraphics[height=2.5cm]{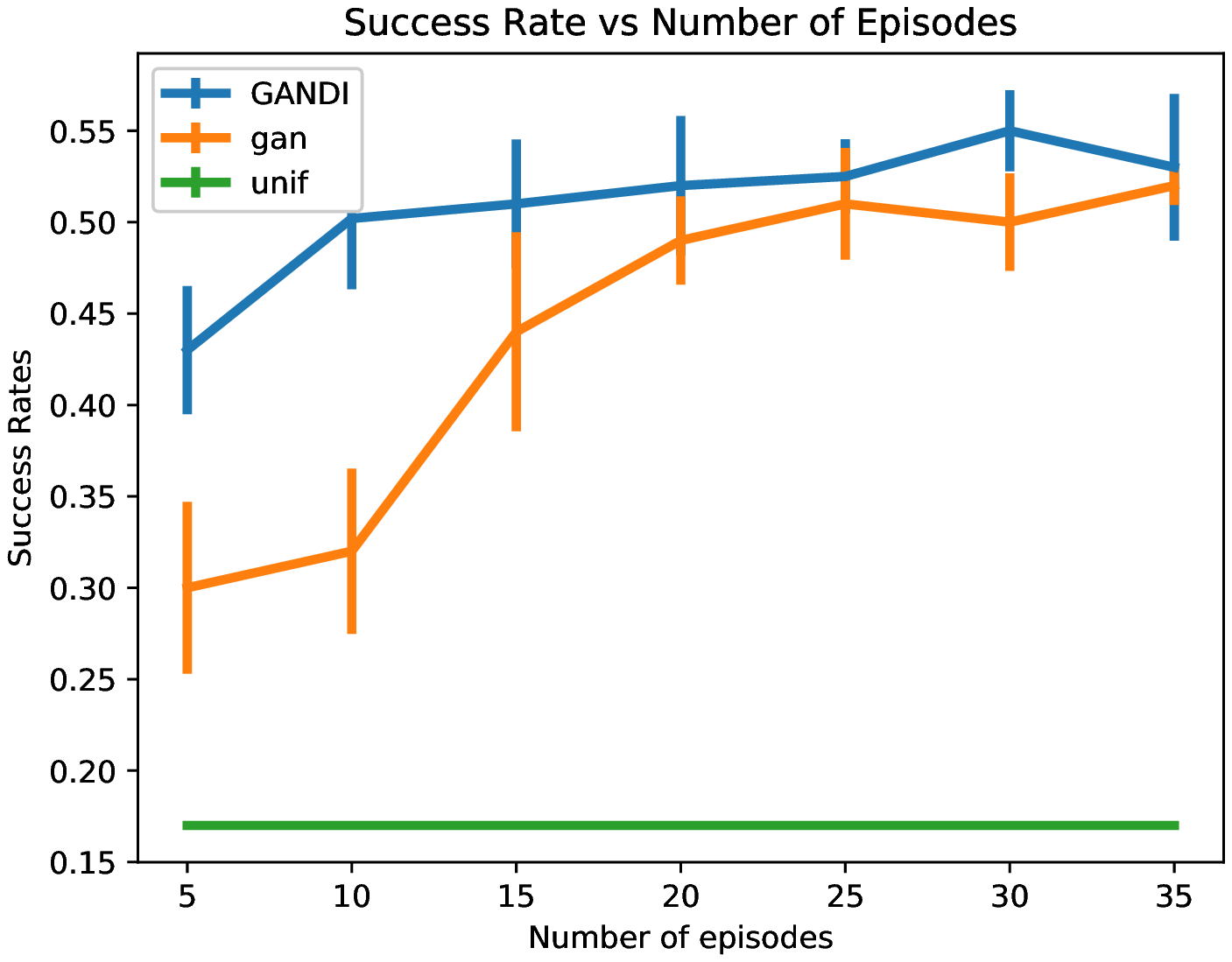}
    \caption{The stow domain }
		\label{fig:stow_rate}
	\end{subfigure} 
	\begin{subfigure}[b]{0.24\textwidth}	
		\centering
		\includegraphics[height=2.5cm]{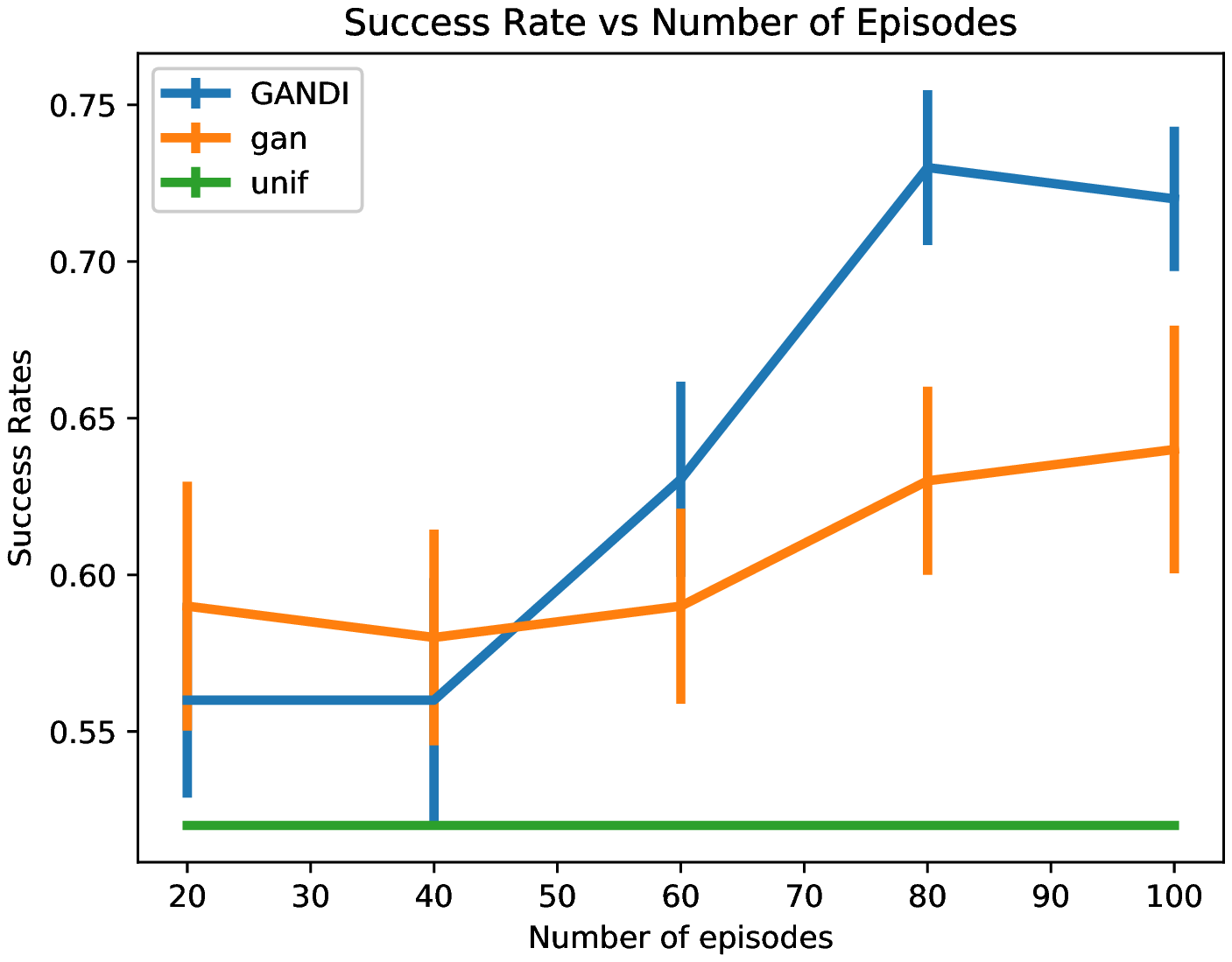}
    \caption{The reconfiguration domain}
		\label{fig:rearrangement_rate}
	\end{subfigure} 
	\caption{Plots of success rate vs. number of training episodes for all three tasks}
\end{figure*} 
\subsection{Bin packing problem}
In this task, a robot has to move 5, 6, 7 or 8 objects into
a region of size 0.3m by 1m (width by length). 
The size of each object is uniformly random between 0.05m to 0.11m,
depending on how many objects the robot has to pack. 
A problem instance is defined by the
number of objects and the size of each object, 
$\prob = [n_{obj},O_{size}]$. A state is defined by the
object placements. For a given problem instance,
all objects have the same size. An example of a solved problem 
instance with $n_{obj}=5$ and $O_{size}=0.11$m is given
 in Figure~\ref{fig:deadend_ex} (right). 

The action space consists of the two dimensional (x,y) 
locations of objects inside the bin, and a uniform action sampler
uniformly samples these values from this region. The robot 
base is fixed. The planning depth varies from 5 to 8, depending
on how many objects need to be placed. This means that
plans consist of 10 to 16 decision variables. 
Clearly, there is a large of volume
of dead-end actions: any action that puts objects down in the 
front of the bin early will lead to a dead-end state.

Figure~\ref{fig:bin_packing_rate} shows the comparison among
GANDI, standard GAN, and uniform action sampling in terms of
success rate when given 5.0 seconds to solve a problem instance. 
We can see the data efficiency of GANDI: with 
20 training episodes, it outperforms the uniform sampler, 
while a standard GAN requires 50 training
episodes to do so. The uniform sampler can only solve
 about 50\% of the problem instances within this time limit, 
while GANDI can solve more than 70\%.

We also compare the distributions for GAN and
GANDI when the same number of training data are given.
Figures~\ref{fig:gan_ex} and~\ref{fig:gandi_ex} show
1000 samples  from GAN and GANDI for packing 5 objects. 
While GANDI learns to avoid the front-middle locations, GAN is
still close to a uniform distribution, and has
a lot of samples in this region which lead to dead-end
states. GANDI focuses its samples to the corners at the back or
the front so that it has spaces for all 5 objects.

\subsection{Stowing objects into crowded bins}
In this task, a robot has to stow 8 objects into 
three different bins, where there already are 10 obstacles. 
A bin is of size 0.4m by 0.3m, an obstacle
is of size 0.05m by 0.05m, and the objects to be placed down
are all of size 0.07m by 0.07m. A problem instance is defined
by the (x,y) locations of 10 obstacles, each of which is randomly
distributed in one of the bins. 
Figure~\ref{fig:stow_ex} shows an instance of a solved stow problem.

The action space for this problem 
consists of (x,y) locations of an object to be placed, 
and the robot's 
(x,y) base pose. This makes 4 dimensional continuous action-space.
The planning depth is always 8, for placing 8 objects.
Thus plans consist of of 36 continuous decision variables.
Again, there is a large volume of dead-end actions, 
similarly to the previous problem: 
putting objects down without consideration of poses of
 later objects can potentially block collision-free 
paths for placing them.

Figure~\ref{fig:stow_rate} compares the success rates of 
the algorithms with a 30-seconds time limit for planning. For the uniform
sampler, we sample first an object placement pose, and then
sample a base pose that can reach the object at its new 
location without colliding with other objects. Unlike the 
previous task, learning-based approaches significantly
outperform the uniform sampling approach for this task.
This is because there is an unstated constraint between
the object placement location and base pose, which is that
the location must be within a reachable distance from
the sampled robot base pose. 
Again, we can observe the data efficiency of GANDI compared to GAN.
When the number of training data points is small, it outperforms it.

\subsection{Reconfiguration planning in a tight space}
In this task, a robot has to reconfigure movable obstacles out of the
way in order to find a collision-free inverse-kinematics solution for its 
left-arm to reach the target object. 
There are five movable obstacles in this problem,
each with size 0.05m by 0.05m, and the target object of size 0.07m by 0.07m, 
and the reconfiguration must happen within a bin, which is of size
0.7m by 0.4m. A problem instance is defined by (x,y) locations 
of the movable obstacles and the target object.  The movable obstacles are randomly 
distributed within the bin; the target object location is distributed along the back of the bin.
Figure \ref{fig:rearrangement_ex} shows an example of a rearrangement 
problem instance at its initial state, with the black region indicating
the distribution of target object locations.

An action specification consists of one discrete value and two continuous values:
what object to move and the (x,y) placement pose of the object being moved.
There is no fixed depth. For both the uniform random sampler and 
the learned generative model, we uniformly at random choose an object to move. 
The robot base is fixed, and the robot is not allowed to use its right arm.

Figure~\ref{fig:rearrangement_rate} compares the success rates of the
algorithms with a 10-seconds time limit for planning. In this problem, 
the learning based approaches outperform the uniform sampler 
even with a small number of training data points. The relationship 
between GANDI and GAN is similar to the previous experiment, except 
that GANDI and GAN are within the each other's confidence interval
when a small number of training points are used. Eventually,
GANDI comes to clearly outperform GAN.

We would like to know if GANDI's distribution indeed
assigns low probabilities to no-progress actions. 
In Figure~\ref{fig:rearrangement_ex},
we show GANDI's distribution of obsject placements after training on 35 episodes. 
The left top corner of the bin is empty because there are no
collision-free inverse-kinematics solution for that region
\footnote{ The robot's left arm will collide
with the bin}. As the figure shows, there are no placement samples 
in front of the target object, but only on the sides that 
 would contribute to clearing space for the robot's left arm to reach
the target object.

\section{Conclusion}
We presented GANDI, a generative model learning
algorithm that uses both on-target and
cheaper off-target samples for data efficiency using importance-ratio estimation.
We provided guarantees on how the error in importance-ratio estimation affects the
performance of the learned model, and  illustrated
that this learning algorithm is effective for learning an 
action sampling distribution for guiding the search of a 
planner in difficult robot planning problems.

\newpage
\bibliographystyle{aaai}
\bibliography{references}
\end{document}


\maketitle
\section{ Proof for theorems}
We first have the following lemma, that readily follows
from the Proposition 1 and Theorem 1 of Goodfellow et al~\cite{GoodfellowNIPS2014}.
\begin{lemma}
For a fixed generator $G$, the optimal discriminator $\hat{\D}$
 that achieves the minimum of equation (6) is
$$ \hat{\D}^*(\x) = \frac{\hatw(\x)q(\x)}{\hatw(\x)q(\x) + \pg(\x)}$$
and the $\pg$ that achieves the maximum of of $\hat{D}^*$ is
$\pg = \hatw(\x)q(\x)$
\end{lemma}

Recall the definition of $\rho$, $$\rho = \sup_{x \in \suppp} \frac{q(\x)}{p(\x)}$$

We first prove the upper-bound on KL divergence between $p$ and $\pg$.
\begin{theorem} If $w(\x) \geq \epsilon\ \forall x\in \suppq$,
 and $J(\hatw) \leq \epsilon^2$, then 
$$ \KL(p||\pg) \leq \log \Big(\frac{1}{1-\epsilon \rho}\Big)$$ 
\begin{proof}
\begin{align*}
\KL(p||\pg) &= \int_\x p(\x) \log \frac{p(\x)}{\hatw(\x)q(\x)}d\x\\
\end{align*}
Since $J(\hatw)\leq \epsilon^2$, we have
 $|\hatw(\x) - w(\x)|\leq \epsilon$, and
 $\forall \x\in \suppq$, $\hat{w}(\x) \geq w(\x) -
\epsilon$. Decreasing the denominator increases
the value of the whole expression, we may take $ w(\x)-\epsilon$ 
to upper bound the $\KL(p||\pg)$. 
\begin{align*}
\KL(p||\pg) &\leq \int_\x p(\x) \log \frac{p(\x)}{q(\x)(w(\x)-\epsilon)}\\
&= -\int_\x p(\x) \log \frac{q(\x)w(\x) - q(\x)\epsilon}{p(\x)}\\
&= -\int_\x p(\x) \log \Big( 1 - \frac{q(\x)}{p(\x)}\epsilon \Big)\\
&\leq  -\int_\x p(\x) \log \Big( 1 - \rho\epsilon \Big)
\end{align*}
where the last line follows from the definition of $\rho$, and decreasing the quantity inside
log increases the value of the whole expression. Now this is equivalent to
\begin{align*}
\KL(p||\pg) \leq -log(1-\rho\epsilon)
\end{align*}
Notice that $0\leq\rho \epsilon\leq 1$ because 
$\rho = \sup_{ x \in \suppp} \frac{1}{w(\x)}$ and
$\epsilon < w(\x)$. 
\end{proof}
\end{theorem}

We now prove the upper-bound on reverse-KL divergence between $p$ and $\pg$.
\begin{theorem}
If $J(\hatw) \leq \epsilon^2$, then
\[KL(p_G || p) \leq (1 + \epsilon) \log (1 + \epsilon\rho) \]
\end{theorem}

\begin{proof}
By Lemma 1,  $p_G(\x) = \hat{w}(\x) q(\x)$, so
\[KL(p_G || p) = \int_\x \hat{w}(\x) q(\x) \log \frac{\hat{w}(\x) q(\x)}{p(\x)} d\x
\]
 Since $J(\hatw)\leq \epsilon^2$, we have
 $|\hatw(\x) - w(\x)|\leq \epsilon$, and
 $\forall x\in \suppq$, $\hat{w}(\x) \leq w(\x) +
\epsilon$. So,
\begin{align*}
KL(p_G \mid p) & \leq \int_\x (w(\x) + \epsilon)q(\x) \log \frac{(w(\x) + \epsilon)
    q(\x)}{p(\x)} d\x \\
& \leq \int_\x (w(\x)q(\x) + \epsilon q(\x)) \log \frac{w(\x)q(\x) + \epsilon
    q(\x)}{p(\x)} d\x \\
& \leq \int_\x (p(\x) + \epsilon q(\x)) \log \frac{p(\x) + \epsilon q(\x)}{p(\x)} d\x \\
& \leq \int_\x (p(\x) + \epsilon q(\x)) \log \left(1 +
  \epsilon\frac{q(\x)}{p(\x)}\right) d\x\\ 
& \leq \int_\x (p(\x) + \epsilon q(\x)) \log (1 + \epsilon\rho) d\x\\
&  \leq \log (1 + \epsilon\rho) \int_\x) (p(\x) + \epsilon q(\x)) d\x\\
&  \leq  (1 + \epsilon) \log (1 + \epsilon\rho)\\
\end{align*}
\end{proof}

\section{Details of experiments}
We have used Keras\footnote{\url{https://github.com/fchollet/keras}}, a deep learning software package, to
implement our NNs.

\subsection{ Architecture for experiments}
For the first domain, $D$, $G$, and $\hatw$ takes in  as input 
the number of objects to be placed and the shape of the object.
They do not conditioned on a state.
For the second domain, they take poses of objects as an input.
They do not condition on the objects already placed. For the last
domain, they take poses of all the objects on the table as input.

For the first domain, we use 3 layers of dense network for $D$,$G$, and $\hatw$.
For the second and last domain, we use convolutional layers, as the location of each
pose in the vector is unimportant.
Specifically, for $D$, we have a input convolutional layer that has filter 
size 2 by 2, with stride of 2 by 1, and 256 filters. This is followed by
a max pooling layer whose size is 2 by 1, and then this conv-max pooling is
repeated one more time. Then, it is followed by two more conv layers whose
number of filters is 256, and the filter size is 3 by 1. Then, the output
gets merged with the action input, which is then
followed by three dense layers each with size 32, 256, and 32. Lastly,
It uses sigmoid activation in its output layer.

The generator has exactly the same architecture, except that the last 
three dense layers have 32, 32, and 32 as its number of nodes. It has
linear output layer. $\hatw$ has exactly the same architecture as the generator

The maximum number of epochs for training $D$ and $G$ is 500, with 
batch size of 32. We used Adam with learning rate 0.001 for $D$ and $G$,
and Adadelta with learning rate 1.0 for $\hatw$.
\bibliographystyle{aaai}
\bibliography{references}